\documentclass{article}

\usepackage{arxiv}
\usepackage{natbib}

\usepackage{algorithmic}

\usepackage{hyperref}
\usepackage{url}

\usepackage[utf8]{inputenc} % allow utf-8 input
\usepackage[T1]{fontenc}    % use 8-bit T1 fonts
\usepackage{hyperref}       % hyperlinks
\usepackage{url}            % simple URL typesetting
\usepackage{booktabs}       % professional-quality tables
\usepackage{amsfonts}       % blackboard math symbols
\usepackage{nicefrac}       % compact symbols for 1/2, etc.
\usepackage{microtype}      % microtypography
\usepackage{xcolor}         % colors
\usepackage[ruled,vlined]{algorithm2e}
\SetKw{KwBy}{by}
\usepackage{lipsum}		% Can be removed after putting your text content
\usepackage{graphicx}
\usepackage{natbib}
\usepackage{doi}
\usepackage{amsthm}
\newtheorem{example}{Example}

\title{Ensembling Graph Predictions for AMR Parsing}

\author{Hoang Thanh Lam$^{1}$, Gabriele Picco$^{1}$, Yufang Hou$^{1}$, Young-Suk Lee$^{2}$, \\ \textbf{Lam M. Nguyen}$^{2}$, \textbf{Dzung T. Phan}$^{2}$, \textbf{Vanessa López}$^{1}$, \textbf{Ramon Fernandez Astudillo}$^{2}$ \\
$^{1}$ IBM Research, Dublin, Ireland \\
$^{2}$ IBM Research, Thomas J. Watson Research Center, Yorktown Heights, USA\\
\texttt{t.l.hoang@ie.ibm.com}, \texttt{gabriele.picco@ibm.com}, \texttt{yhou@ie.ibm.com}, \\ 
\texttt{ysuklee@us.ibm.com}, \texttt{LamNguyen.MLTD@ibm.com}, \texttt{phandu@us.ibm.com},   \\
\texttt{vanlopez@ie.ibm.com}, \texttt{ramon.astudillo@ibm.com}
}

\usepackage{pifont}

\usepackage{epsfig}
\usepackage{amssymb}
\usepackage{amsmath}
\usepackage{amsthm}
\usepackage{amsfonts}
\usepackage{bbding}
\usepackage{array}
\usepackage{paralist}
\usepackage{xargs}                      % Use more than one optional parameter in a new commands
\usepackage{caption}

\newtheorem{theorem}{Theorem}

\newtheorem{problem}{Problem}

\newcolumntype{C}[1]{>{\centering\let\newline\\\arraybackslash\hspace{0pt}}m{#1}}

\usepackage{xcolor}

\newcommand{\zero}[1]{{\boldsymbol{0}}}

\begin{document}

\maketitle

\begin{abstract}
In many machine learning tasks, models are trained to predict structure data such as graphs. For example, in natural language processing, it is very common to parse texts into dependency trees or abstract meaning representation (AMR) graphs. On the other hand, ensemble methods combine predictions from multiple models to create a new one that is more robust and accurate than individual predictions. In the literature, there are many ensembling techniques proposed for classification or regression problems, however, ensemble graph prediction has not been studied thoroughly. In this work, we formalize this problem as mining the largest graph that is the most supported by a collection of graph predictions. As the problem is NP-Hard, we propose an efficient heuristic algorithm to approximate the optimal solution. To validate our approach, we carried out experiments in AMR parsing problems. The experimental results demonstrate that the proposed approach can combine the strength of state-of-the-art AMR parsers to create new predictions that are more accurate than any individual models in five standard benchmark datasets\footnote{This ArXiv paper is an updated version of the published paper at NeurIPS 2021 \url{https://openreview.net/forum?id=lmm2W2ICtjk}. The paper discusses related work by \cite{barzdins} in Section \ref{sec:related work} and provides detailed discussion and comparative studies in Subsection \ref{subsec:comparison}.  }. 
\end{abstract}

\section{Introduction}\label{sec_intro}

Ensemble learning is a popular machine learning practice, in which predictions from multiple models are blended to create a new one that is usually more robust and accurate. Indeed, ensemble methods like XGBOOST are the winning solution in many machine learning and data science competitions \citep{xgboost}. A key reason behind the successes of the ensemble methods is that they can combine the strength of different models to reduce the variance and bias in the final prediction \citep{domingos2000unified,valentini2004bias}.  Research in ensemble methods mostly focuses on regression or classification problems \citep{dong2020survey}. Recently, in many machine learning tasks prediction outputs are provided in a form of graphs. For instance, in Abstract Meaning Representation (AMR) parsing  %problem 
\citep{banarescu2013abstract}, the input is a fragment of text and the output is a  rooted, labeled, directed, acyclic graph (DAG). It abstracts away from syntactic representations, in the sense that 
%similar meaning sentences 
sentences with similar meaning
should have the same AMR. Figure \ref{fig:amr} shows an AMR graph for the sentence \emph{You told me to wash the dog} where nodes are concepts and edges are relations.

AMR parsing is an important problem in natural language processing (NLP) research %having 
and it has
a broad application in downstream tasks such as question answering \citep{kapanipathi2020question} and common sense reasoning \citep{lim-etal-2020-know}. Recent approaches for AMR parsing leverage the advances from pretrained language models \citep{spring} and numerous deep neural network architectures \citep{cai-lam,zhou2021apt}. 

 %Different from
 Unlike methods for ensembling numerical or categorical values for regression or classification problems where the mean value or majority votes are used respectively, the problem of graph ensemble is more complicated. For instance, Figure \ref{fig:running_example} show three graphs $g_1, g_2, g_3$ with different structures, having varied number of edges and vertices with different labels. In this work, we formulate the ensemble graph prediction as a graph mining problem where we look for the largest common structure among the graph predictions. In general, finding the largest common subgraph is a well-known computationally intractable problem in graph theory. However, for AMR parsing problems where the AMR graphs have labels and a simple tree-alike structure, we %can
 propose an efficient heuristic algorithm (\emph{Graphene}) to approximate the solution of the given problem well.
 
\begin{figure*}[htb]
\center{\includegraphics[width=\textwidth]{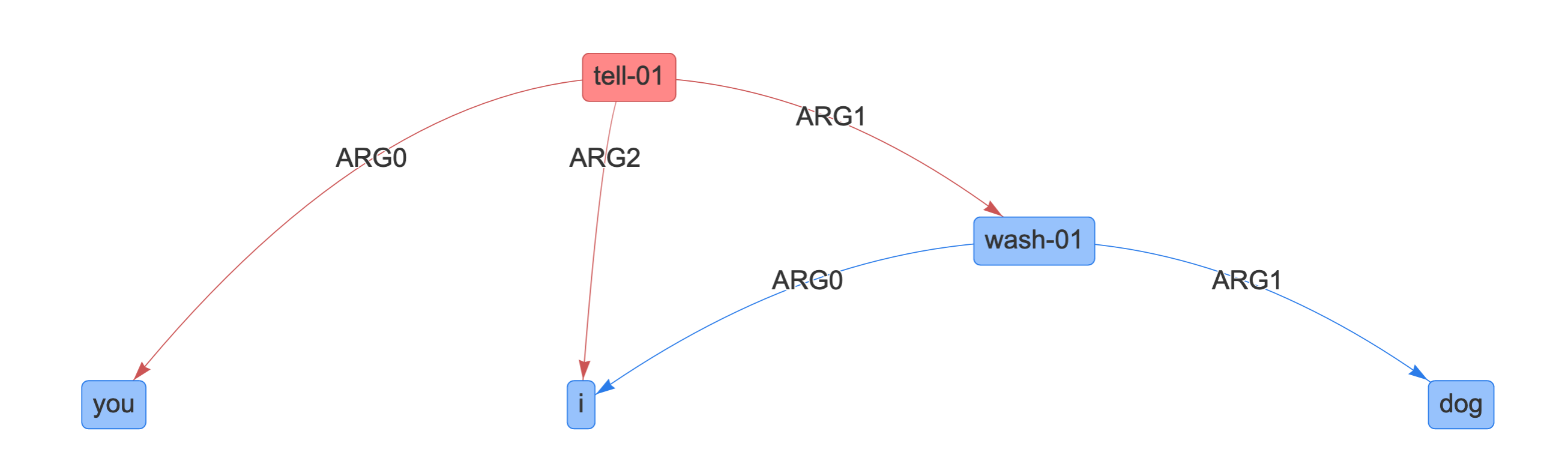}}
\caption{\label{fig:amr} An example AMR graph for the sentence \emph{You told me to wash the dog}.}
\end{figure*}
 
To validate our approach, we collect the predictions from four state-of-the-art AMR parsers and create new predictions using the proposed graph ensemble algorithm.  The chosen AMR parsers are the recent state-of-the-art AMR parsers including a seq2seq-based method using  BART \citep{spring}, a transition-based approach proposed in \citep{zhou2021apt} and a graph-based approach proposed in \citep{cai-lam}. In addition to those models, we also trained a new seq2seq model based on T5 \citep{t5} to leverage the strength of this pretrained language model.% T5.  

The experimental results show that in five standard benchmark datasets, our proposed 
ensemble approach outperforms the previous state-of-the-art models and achieves new state-of-the-art results in all datasets. For example, our approach achieves new state-of-the-art results with 1.7, 1.5, and 1.3 points better than prior arts in the BIO (under out-of-distribution evaluation),  AMR 2.0, and AMR 3.0 datasets respectively. This result demonstrates the strength of 
%the ensemble methods 
our ensemble method 
in leveraging the model diversity to achieve better performance. An interesting property of our solution is that it is model-agnostic, therefore it can be used to make an ensemble of existing model predictions without the requirement to have an access to model training. Source code is open-sourced\footnote{\url{https://github.com/IBM/graph_ensemble_learning}}.

Our paper is organized as follows: Section \ref{sec:problem definition} discusses a formal problem definition and a study on the computational intractability of the formulated problem. The graph ensemble algorithm is described in Section \ref{sec:algorithm}. Experimental results are reported in Section \ref{sec:experiments} while Section \ref{sec:related work} discusses related works. The conclusion and future work are discussed in Section \ref{sec:conclusions and future work}.
\section{Problem formulation}
\label{sec:problem definition}
Denote $g = (E, V)$ as a graph with the set of edges $E$ and the set of vertices $V$. Each vertex $v \in V$ and edge $e \in E$ is associated with a label denoted as $l(v)$ and $l(e)$ respectively, where $l(.)$ is a labelling function. Given two graphs $g_1 = (E_1, V_1)$ and $g_2 = (E_2, V_2)$, a \emph{vertex matching} $\phi$ is a bijective function that maps a vertex $v \in V_1$ to a vertex $\phi(v) \in V_2$. 

\begin{example}
In Figure \ref{fig:running_example}, between $g_1$ and $g_2$ there are many possible vertex matches, where $\phi(g_1, g_2) = [1 \rightarrow 3, 2 \rightarrow 2, 3 \rightarrow 1]$ is one of them (which can be read as the first vertex of $g_1$ is mapped to the third vertex of $g_2$ and so forth). Notice that not all vertices $v \in V_1$ has a match in $V_2$ and vice versa. Indeed, in this example, the fourth vertex in $g_2$ does not have a matched vertex in $g_1$.
\end{example}

\begin{figure*}[htb]
\center{\includegraphics[width=\textwidth]{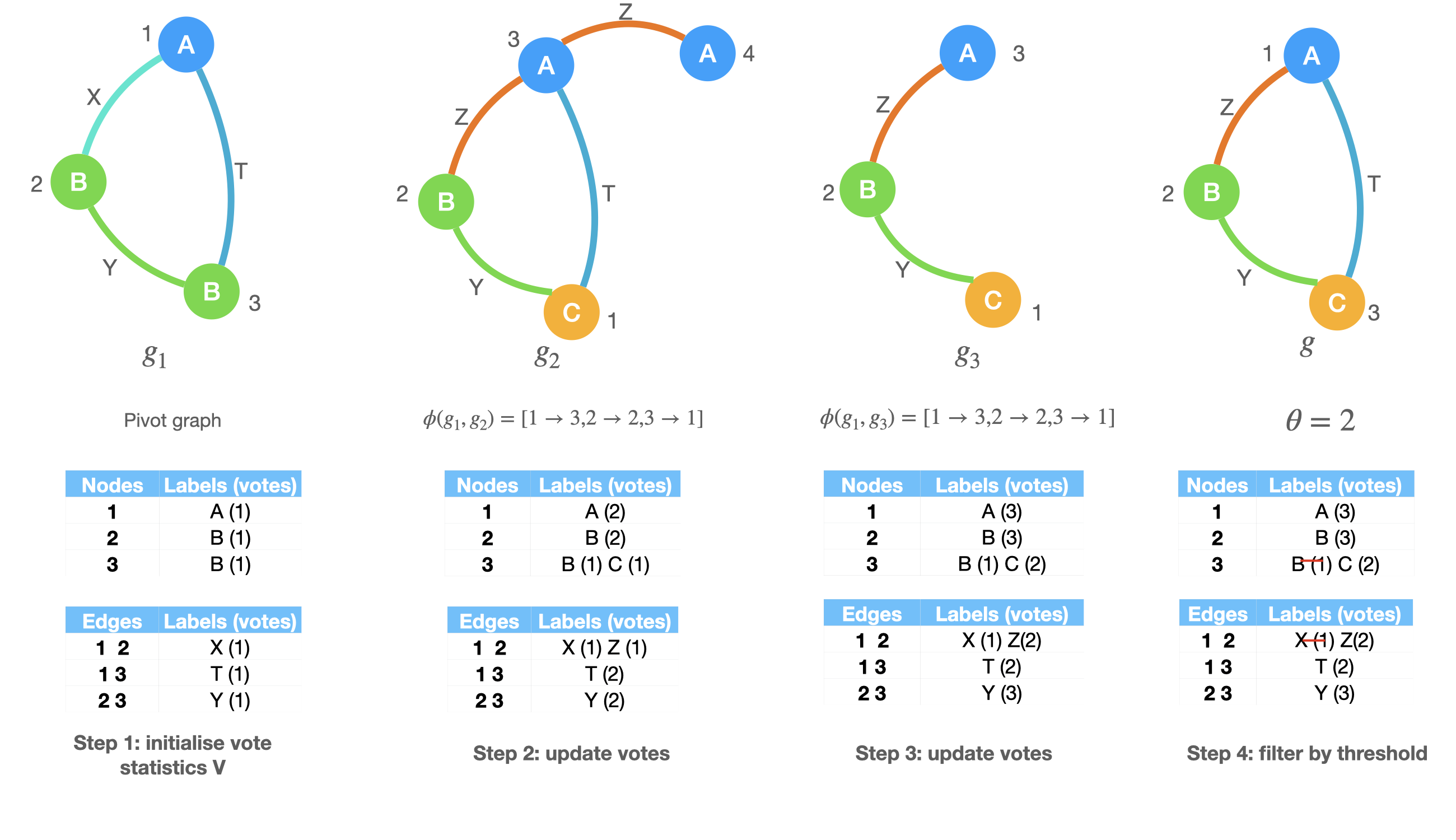}}
\caption{\label{fig:running_example} A graph ensemble example.  Each node and edge of $g$ occurs in at least two out of three graphs $g_1, g_2, g_3$. Therefore, $g$ is  $\theta$-supported  where $\theta=2$ by the given set of graphs. Graph $g$ is also the graph with the largest sum of supports among all $\theta$-supported graphs. The tables show the node and edge support (votes) are updated in each step of the Graphene algorithm when $g_1$ is a pivot graph.}
\end{figure*}

Given two graphs $g_1$, $g_2$ and a vertex match $\phi(g_1, g_2)$, support of a vertex $v$ with respect to the matching $\phi$, denoted as $s_{\phi}(v)$, is equal to 1 if $l(v) = l(\phi(v))$ and 0 otherwise. Given  an edge $e = (v_1, v_2)$ the support of $e$ with respect to the vertex match $\phi$, denoted as $s_{\phi}(e)$, is equal to 1 if $l(e) = l((\phi(v_1), \phi(v_2)))$ and 0 otherwise.

\begin{example}
In Figure \ref{fig:running_example}, for the  vertex match $\phi(g_1, g_2) = [1 \rightarrow 3, 2 \rightarrow 2, 3 \rightarrow 1]$, the first vertex in $g_1$ and the third vertex in $g_2$ shares the same label $A$, therefore the support of the given vertex is equal to 1. On the other hand, the third vertex in $g_1$ and the first vertex in $g_2$ does not have the same label so their support is equal to 0.
\end{example}

Between two graphs, there are many possible vertex matches, the \emph{best vertex match} is defined as the one that has the maximal total vertex support and edge support. In our discussion, when we mention a vertex match we always refer to the best vertex match. 

Denote $G = \{g_1=(E_1, V_1), g_2=(E_2, V_2),\cdots, g_m=(E_m, V_m)\}$ as a set of $m$ graphs. Given any  graph $g=(E, V)$, for every $g_i$ denote $\phi_i(g, g_i)$ as the best vertex match between $g$ and $g_i$. The total support of a vertex $v \in V$ or an edge $e \in E$ %as
is defined as follows:
\begin{itemize}
    \item $\operatorname{support}(e) = \sum_{i=1}^m s_{\phi_i}(e)$
    \item $\operatorname{support}(v) = \sum_{i=1}^m s_{\phi_i}(v)$
\end{itemize}

Given a support threshold $\theta$, a graph $g$ is called $\theta$-supported by $G$ if for any node $v \in V$ or any edge $e \in E$,  $\operatorname{support}(v) \geq \theta$ and $\operatorname{support}(e) \geq \theta$.

\begin{example}
In Figure \ref{fig:running_example}, graph $g$ is  $\theta$-supported by $G =\{g_1, g_2, g_3\}$ where $\theta=2$.
\end{example}

Intuitively, an ensemble graph $g$ should have as many common edges and vertices with all the graph predictions as possible. Therefore, we define the graph ensemble problem as follows:

\begin{problem}[Graph ensemble]
Given a support threshold $\theta$ and a collection of graphs $G$, find the graph $g$ that is $\theta$-supported by $G$ and has the largest sum of vertex and edge supports. 
\label{problem:graph ensemble}
\end{problem}

\begin{theorem}
Finding the optimal $\theta$-supported graph with the largest total of support is NP-Hard.
\label{theorem:np-hard}
\end{theorem}
\begin{proof}
We prove the NP-Hardness by reduction to the \emph{Maximum Common Edge Subgraph} (MCES) problem, which is known to be an NP-Complete problem \citep{npcomplete}. Given two graphs $g_1$ and $g_2$, the MCES problem finds a graph $g$ that is a common subgraph of  $g_1$ and $g_2$ and the number of edges in $g$ is the largest. Consider the following instance of the Graph Ensemble problem with $\theta = 2$, and $G= \{g_1, g_2\}$ created from the graphs in the MCES problem. Assume that all vertices and all edges of $g_1$ and $g_2$ have the same label $A$. 

Since $\theta=2$, a  $\theta$-supported graph is  also a common subgraph between $g_1$ and $g_2$ and vice versa. Denote $g_s$ and $g_e$ as the common subgraph between $g_1$ and $g_2$ with the largest support and the  largest common edge, respectively. We can show that $g_s$ has as many edges as $g_e$. In fact, since $g_s$ is the largest supported common subgraph there is no vertex  $v \in g_e$ such that $v \not\in g_s$ because otherwise we can add $v$ to $g_s$ to create a larger supported graph. For any edge $e = (v_1, v_2) \in g_e$, since both vertices $v_1$ and $v_2$ also appear in $g_s$, the edge $e = (v_1, v_2)$ must also be part of $g_s$ otherwise we can add this edge to $g_s$ to create a subgraph with a larger support. Therefore, $g_s$ has as many edges as $g_e$, which is also a solution to the MCES problem. 
 \end{proof}
 
\section{Graph ensemble algorithm}
\label{sec:algorithm}
In this section, we discuss a heuristic algorithm based on the strategy \emph{``Please correct me if I am wrong!"} to solve Problem \ref{problem:graph ensemble}. The main idea is %very simple,
to improve a pivot graph based on other graphs. Specifically, 
starting with a pivot graph $g_i$ ($i=1,2,\cdots,m$), we collect the votes from the other graphs for every existing vertex and existing/non-existing edges to correct $g_i$. We call the proposed algorithm \emph{Graphene} which stands for Graph Ensemble algorithm. The key steps of the algorithm are provided in the pseudo-code in Algorithm \ref{alg:graphene}. 

\begin{algorithm}[H]
\SetAlgoLined
\textbf{Input}: a set of graphs $G = \{g_1, g_2,\cdots, g_m\}$ and the support threshold $\theta$ \\
\textbf{Output}: an ensemble graph $g^e$\\
\textbf{Algorithm}: Graphene($G, \theta$)\\
 \For{$i \gets 1$ \KwTo  $m$}{
    $g_\text{pivot} \leftarrow g_i$ \\
    $V \leftarrow$ Initialise($g_\text{pivot}$) \\
    \For{$j \gets 1$ \KwTo  $m$}{
        \uIf {$j \ne i$}{
         $V  \leftarrow V$ $\cup$ getVote($\phi(g_\text{pivot}, g_j$))
        }
    }
    $g^e_i \leftarrow Filter(V, \theta)$
 }
 
 $g^e \leftarrow $ the graph with the largest support among $g^e_1,\cdots,g^e_m$\\
 \textbf{Return} $g^e$
 
 \caption{Graph ensemble with the Graphene algorithm.}
 \label{alg:graphene}
\end{algorithm}

For example, in Figure \ref{fig:running_example}, the algorithm starts with the first graph $g_1$ and considers it as a pivot graph $g_\text{pivot}$. In the first step, it creates a table to keep voting statistics $V$ initialized with the vote counts for every existing vertex and edge in $g_\text{pivot}$. To draw additional votes from the other graphs, it performs the following subsequent steps:
\begin{itemize}
    \item Call the function $\phi(g_1, g_i)$ $(i=2,3,\cdots,m)$ to get the best  bijective mapping $\phi$ between the vertices of two graphs $g_1$ and $g_i$ (with a little bit abuse of notation we drop the index $i$ from $\phi_i$ when $g_i$ and $g_\text{pivot}$ are given in the context). For instance, the best vertex match between $g_1$ and $g_2$ is $\phi = 1 \rightarrow 3, 2 \rightarrow 2, 3 \rightarrow 1$ because that vertex match has the largest number of common labeled edges and vertices.
    \item Enumerate the matching vertices and edges to update the voting statistics accordingly. For instance,  since the vertex 3 in $g_1$ with label $B$ is mapped to the vertex 1 in $g_2$ with label $C$, a new candidate label $C$ is added to the table for the given vertex. For the same reason, we add a new candidate label $Z$ for the edge $(1,2)$. For all the other edges and vertices where the labels are matched the votes are updated accordingly.
\end{itemize}
Once the complete voting statistics $V$ is available, the algorithm filters the candidate labels of edges and vertices using the provided support threshold $\theta$ by calling the function $Filter(V, \theta)$ to obtain an ensemble graph $g^e_i$. For special cases, when disconnected graphs are not considered as a valid output, we keep all edges of the pivot graph even its support is below the threshold.  On the other hand, for the graph prediction problem, where a graph is only considered a valid graph if it does not have multiple edges between two vertices and multiple labels for any vertex,  we remove all candidate labels for vertices and edges  except the one with the highest number of votes.

Assume that the resulting ensemble graph that is created  by using $g_i$ as the pivot graph is denoted as $g^e_i$. The final ensemble graph $g^e$ is chosen among the set of graphs $g^e_1, g^e_2, \cdots, g^e_m$ as the one with the largest total support. Recall that $\phi(g_\text{pivot}, g_i)$ finds the best vertex match between two graphs. In general, the given task is computationally intractable. However, for labeled graphs like AMR a heuristic was proposed \citep{cai-knight-2013-smatch} to approximate the best match by a hill-climbing algorithm. It first starts with the candidate with labels that are mostly matched. The initial match is modified iteratively to optimize the total number of matches with a predefined number of iterations (default value set to 5). This algorithm is very efficient and effective, it was used to calculate the Smatch score in \citep{cai-knight-2013-smatch} so we reuse the same implementation to approximate  $\phi(g_\text{pivot}, g_i)$ (report on average running time can be found in the supplementary materials).
%\begin{theorem}
%The Graphene algorithm is an approximation algorithm of Problem \ref{problem:graph ensemble} with the approximation factor [To be determined].
%\label{theorem:approximation}
%\end{theorem}

\section{Experiments}
\label{sec:experiments}
We compare our \emph{Graphene} algorithm against four previous state-of-the-art models on different benchmark datasets. Below we describe our experimental settings.

%Four state-of-the-art models are used as a comparison and to evaluate the performance of the Graphene algorithm. The experimental settings are described below.

\subsection{Experimental settings}
\label{sec:settings}
\subsubsection{Model settings}

\paragraph{SPRING}
The SPRING model, presented in \citep{spring}, tackles Text-to-AMR and AMR-to-Text as a symmetric transduction task. The authors show that with a pretrained encoder-decoder model, it is possible to obtain state-of-the-art performances in both tasks using a simple seq2seq framework
%and 
by
predicting %or encoding
linearized graphs. In our experiments, we used the pretrained models provided in \citep{spring}\footnote{Available for download at \url{https://github.com/SapienzaNLP/spring}}. In addition, we trained 3 more models using different random seeds following the same setup described in \citep{spring}. Blink \citep{blink} was used to add wiki tags to the predicted AMR graphs as a post-processing step.  

\paragraph{T5}
The T5 model, presented in \citep{t5}, introduces a unified framework that 
%considers every language problem 
models a wide range of NLP tasks 
as a text-to-text problem. We follow the same idea proposed in \citep{xu2020improving} to train a model to transfer a text to a linearized AMR graph based on T5-Large.
%train T5 large as a text to AMR problem. 
The data is preprocessed by linearization and removing wiki tags using the script provided in \citep{amrlib}. In addition to the main task, we added a new task that takes as input a sentence and predicts the concatenation of word senses and arguments provided in the English Web Treebank dataset \citep{google}. The model is trained with 30 epochs. We use  ADAM optimization with a learning rate of 1e-4 and a mini-batch size of 4. Blink \citep{blink} was used to add wiki tags to the predicted AMR graphs during post-processing.

\paragraph{APT}
\citep{zhou2021apt} proposed a transition-based AMR parser\footnote{Available under https://github.com/IBM/ transition-amr-parser.} %modeled with a single 
based on Transformer \citep{vaswani2017transformer}. It combines hard-attentions over sentences with a target side action pointer mechanism to decouple source tokens from node representations.  For model training in our experiments, we use the setup described in \citep{zhou2021apt} and added 70K model-annotated silver data sentences to the training data, which was created from the 85K sentence set in \citep{youngsuk2020silver} with self-learning described in the paper.

\paragraph{Cai\&Lam}
The model proposed in \citep{cai-lam-2020-amr} treats AMR parsing as a series of dual decisions (i.e., \emph{which parts of the sequence to abstract}, and  \emph{where in the graph to construct}) on the input sequence and constructs the AMR graph incrementally. 
Following \citep{cai-lam-2020-amr}, we use Stanford CoreNLP\footnote{Available at \url{https://github.com/stanfordnlp/stanza/}} for tokenization, lemmatization, part-of-speech tagging, and named entity recognition. We apply the pretrained model provided by the authors\footnote{The model ``AMR2.0+BERT+GR'' can be downloaded from \url{https://github.com/jcyk/AMR-gs}} to all testing datasets and follow the same pre-processing and post-processing steps for graph re-categorization.  

\paragraph{Graphene (our algorithm)}
The only hyperparameter of the Graphene algorithm is the threshold $\theta$. A popular practice for ensemble methods via voting strategy \citep{dong2020survey} is to consider the labels that get at least 50\% of the total number of votes, therefore we set the threshold $\theta$ such that $\frac{\theta}{m} \geq 0.5$ (where $m$ is the number of models in the ensemble). In all experiments, we used a Tesla GPU V100 for model training and used 8 CPUs for making an ensemble.

\subsubsection{Evaluation}
We use the script\footnote{\url{https://github.com/mdtux89/amr-evaluation}} provided in \citep{damonte-etal-2017-incremental} to calculate the Smatch score \citep{cai-knight-2013-smatch}, the most relevant metric for measuring the similarity between the predictions and the gold AMR graphs. The overall Smatch score can be broken down into different dimensions, including the  followings sub-metrics:
\begin{itemize}
    \item Unlabeled (Unl.): Smatch score after removing all edge labels
    \item No WSD (NWSD): Smatch score while ignoring Propbank senses 
    \item NE: F-score on the named entity recognition (:name roles)
    \item  Wikification (Wiki.): F-score on the wikification (:wiki roles)
    \item Negations (Neg.): F-score on the negation detection (:polarity roles)
    \item Concepts (Con.): F-score on the concept identification task
    \item Reentrancy (Reen.): Smatch computed on reentrant edges only
    \item SRL: Smatch computed on :ARG-i roles only
\end{itemize}

\subsubsection{Datasets}

%\textbf{\textcolor{red}{LN: Made changed - Table's caption must be on the top and Figure's caption must be on the bottom. This is the rule of ICML and NeurIPS.}}

Similarly to \citep{spring}, we use five standard benchmark datasets \citep{data} to evaluate our approach. Table \ref{tab:data} shows the statistics of the datasets. AMR 2.0 and AMR 3.0 are divided into train, development and testing sets and we use them for \emph{in-distribution} evaluation in %Subsection
Section
\ref{subsec:in-distribution}. Furthermore, the models trained on AMR 2.0 training data are used to evaluate \emph{out-of-distribution} prediction on the BIO, the LP and the New3 dataset (See %Subsection
Section \ref{subsec:out-of-distribution}).
\begin{table}[h]
\centering
\caption{\label{tab:data}Benchmark datasets. All instances of BIO, LP, and New3 are used to test models in out-of-distribution evaluation. For AMR 2.0 and 3.0, the models are trained on the training dataset, validated on the development dataset. We report results on testing sets in the in-distribution evaluation.}
\begin{tabular}{l|rrrrr}
\hline
\textbf{Datasets} & \textbf{AMR 2.0} & \textbf{AMR 3.0} & \textbf{BIO} & \textbf{Little Prince (LP)} & \textbf{New3} \\ \hline
Training         & 36,521             & 55,635             & n/a         & n/a                        & n/a           \\
Dev         & 1,368             & 1,722             & n/a         & n/a                        & n/a           \\
Test         & 1,371             & 1,898             & 6,952         & 1,562                        & 527           \\ \hline
\end{tabular}
\end{table}
\subsection{In-distribution evaluation}
\label{subsec:in-distribution}
In the same spirit of \citep{spring}, we evaluate the approaches when training and test data belong to the same domain. Table \ref{tab:amr} shows the results of the models on the test split of the AMR 2.0 and AMR 3.0 datasets. The metrics reported for SPRING correspond to the model with the highest Smatch score among the 4 models(the checkpoint plus the 3 models with different random seeds).  

%\textcolor{red}{I used Spacy amrlib  preprocessing scripts to preprocess the AMR 2.0 and 3.0 files. The scripts load the data and turn into  AMR graphs, preprocesses the graphs before outputs the AMR back to text format. The result was calculated using the preprocessed gold files (AMR 2.0 and AMR 3.0) and spring preprocessed gold files (BIO, LP and New3). These gold files are slightly different than the original gold files (see the reason https://github.com/snowblink14/smatch/issues/39). I am re-running all the experiments against the gold files that were created by concatenating the test files manually. The results are not different, but we need to re-run all the experiments to make sure that it is correct. I will put the new results beside the old results and comment out the old ones in this PDF}

\begin{table}[htbp]
\centering
\caption{\label{tab:amr} Results on the test splits of the AMR 2.0 and AMR 3.0 dataset.
%in the in-distribution evaluation.
}
\begin{tabular}{l|lllllllll}
\hline
\textbf{Models} & \textbf{Smatch} & \textbf{Unl.} & \textbf{NWSD} & \textbf{Con.} & \textbf{NE}    & \textbf{Neg.}  & \textbf{Wiki.} & \textbf{Reen.} & \textbf{SRL}   \\ \hline
\textbf{AMR 2.0}  &  &   &   &  &   &  &  &  &  \\
SPRING        & 84.22           & 87.38          & 84.72          & 89.98          & 90.77          & 72.65          & 82.76 & 74.30           & 82.89          \\
APT        & 82.70           & 86.18          & 83.23          & 89.48          & 90.20          & 67.27          & 78.87 & 73.19           & 82.01          \\
T5              & 82.98           & 86.17          & 83.43          & 89.85          & 90.65          & 73.43          & 77.99          & 72.44           & 82.02          \\
Cai\&Lam & 80.15           & 83.60          & 80.66          & 87.39          & 82.25          & \textbf{78.09}          & \textbf{85.36}          & 66.46           & 77.35          \\

\emph{Graphene 4S}     & 84.78  & 87.96  & 85.29 & 90.64 & 92.19 & 75.22 & 83.88          & 71.42            & 83.46 \\
\emph{Graphene Support}     & 85.85  & 88.68  & 86.35 & 91.23 & 92.30 & 77.01 & 84.63          & 74.49            & 84.41 \\
\emph{Barzdins}     & 85.93  & 88.85  & 86.63 & 91.16 & \textbf{92.51} & 75.73 & 84.01          & \textbf{76.53}            & 84.64\\
\emph{Graphene Smatch}     & \textbf{86.26}  & \textbf{89.03}  & \textbf{86.75} & \textbf{91.39} & 92.47 & 76.72 & 84.61          & 76.25            & \textbf{84.85}\\
\hline \hline
\textbf{AMR 3.0}  &  &   &   &  &   &  &  &  &  \\

SPRING       & 83.25           & 86.40          & 83.71          & 89.38          & 87.80          & 72.94          & 81.22          & 73.33  & 81.97          \\

APT        & 80.57           & 83.96          & 81.07          & 88.38          & 86.82          & 68.69          & 76.88 & 70.78           & 80.17          \\
T5              & 82.17           & 85.22          & 82.66          & 89.03          & 86.99          & 72.59          & 73.78          & 72.18           & 81.18          \\
\emph{Graphene 4S}              & 83.77           & 86.89          & 84.23          & 90.09          & 88.27          & 74.60          & 81.92          & 70.22           & 82.46          \\
\emph{Graphene Support}              & 84.41           & 87.35          & 84.83          & 90.51          & 88.64          & 74.76          & \textbf{82.25}          & 71.93           & 83.15          \\
\emph{Barzdins}              & 84.74           & 87.66          & 85.21          & 90.40          & 88.63          & \textbf{74.78}          & 81.49          & \textbf{75.48}           & 83.55          \\
\emph{Graphene Smatch}              & \textbf{84.87}           & \textbf{87.76}          & \textbf{85.31}          & \textbf{90.57}          & \textbf{88.81}          & 74.73          & 82.15          & 74.06           & \textbf{83.72}          \\

\hline
\end{tabular}
\end{table}

For the ensemble approach, we report the result when Graphene is an ensemble of four SPRING checkpoints, denoted as \emph{Graphene 4S}. The ensemble of all the models including four SPRING checkpoints, APT, T5, and Cai\&Lam is denoted as \emph{Graphene All}. For the AMR 3.0 dataset, the Cai\&Lam model is not available so the reported result corresponds to an ensemble of all six models. 

We can see that Graphene successfully leverages the strength of all the models and provides better prediction both in terms of the overall Smatch score and sub-metrics. In both datasets, we achieve the state-of-the-art results with %accuracy
performance gain of 1.6 and 1.2 Smatch points in AMR 2.0 and AMR 3.0 respectively. Table \ref{tab:amr} shows that by combining predictions from four checkpoints of the SPRING model, Graphene 4S provides better results than any individual models. The result is improved further when increasing the number of ensemble models, indeed \emph{Graphene All} improves  \emph{Graphene 4S} further and outperforms the individual models in terms of the overall Smatch score.     

\subsection{Out-of-distribution evaluation}
\label{subsec:out-of-distribution}
In contrast to in-distribution evaluation, we use the models trained with AMR 2.0 data to collect AMR predictions for the testing datasets in the domains that differ from the AMR 2.0 dataset. The purpose of the experiment is to evaluate the ensemble approach under out-of-distribution settings. 

Table \ref{tab:out-of-distribution} shows the results of our experiments. Similar to the in-distribution experiments, the \emph{Graphene 4S} algorithm achieves better results than other individual models, while the \emph{Graphene All} approach improves the given results further. We achieve the new  state-of-the-art results in these benchmark datasets (under out-of-distribution settings). This result has an important practical implication because in practice it is very common not to have labeled AMR data for domain-specific texts. After all, the labeling task is very time-demanding. Using the proposed ensemble methods we can achieve better results with domain-specific data not included in the training sets.

\begin{table}[]
\centering
\caption{\label{tab:out-of-distribution}
Results of out-of-distribution evaluation on the BIO, New3, and Little Prince dataset.}
\begin{tabular}{l|lllllllll}
\hline
\textbf{Models} & \textbf{Smatch} & \textbf{Unl.} & \textbf{NWSD} & \textbf{Con.} & \textbf{NE}    & \textbf{Neg.}  & \textbf{Wiki.} & \textbf{Reen.} & \textbf{SRL}   \\ \hline
\textbf{BIO}  &  &   &   &  &   &  &  &  &  \\
SPRING          & 60.52            & 65.33          & 61.42          & 67.76          & \textbf{33.92} & 65.68          & 3.80           & 51.19           & 62.86          \\
APT       & 51.23           & 56.27          & 51.81          & 58.22          & 15.68          & 52.91          & 3.62  & 43.53           & 54.24          \\
T5              & 58.89           & 63.86          & 59.69          & 66.63          & 30.42          & 65.11          & 2.46           & 48.56           & 61.47          \\
Cai\&Lam & 42.22           & 49.78          & 42.85          & 47.10          & 5.19          & 51.42          & \textbf{7.32}          & 39.23           & 51.00          \\

\emph{Graphene 4S}        & 61.51  & 66.22 & 62.28 & 68.48 & 33.02          & 68.24 & 4.46           & 50.40  & 63.70 \\
\emph{Graphene Support}       & 62.29  & 66.89 & 63.07 & 68.64 & 32.62          & 69.48 & 4.54           & 52.06  & 64.21 \\
\emph{Barzdins}       & 62.05  & 66.84 & 62.88 & 68.43 & 32.73          & 68.11 & 4.23           & 52.22  & 64.08 \\
\emph{Graphene Smatch}       & \textbf{62.80}  & \textbf{67.39} & \textbf{63.58} & \textbf{68.94} & 32.82          & \textbf{69.13} & 4.37           & \textbf{52.58}  & \textbf{64.56} \\
\hline \hline
\textbf{New3}  &  &   &   &  &   &  &  &  &  \\
SPRING          & 74.66           & 78.99          & 75.21          & 82.38          & 67.52          & 67.48          & 67.20          & 66.47  & 75.65          \\

APT        & 71.06           & 75.92          & 71.58          & 80.34          & 65.65          & 67.08          & 57.14 & 63.02           & 73.40          \\
T5              & 73.04            & 77.30 & 73.68          & 82.65          &   68.24        & 64.20          & 56.42          & 64.65           & 75.03 \\
Cai\&Lam              & 60.81            & 66.00 & 61.29          & 72.79          & 45.60          & 59.57          & 46.39          & 57.70           & 68.87 \\\hline
\emph{Graphene 4S}       & 74.84  & 79.23          & 75.30 & 82.56 & 69.98 & 69.51 & \textbf{68.34}          & 63.53           & 76.31          \\
\emph{Graphene Support}      & 75.60  & 79.64          & 76.14 & 83.08 & 68.40 & 69.62 & 67.98          & 67.16           & 76.88         \\
\emph{Barzdins}      & 75.87  & 80.22          & 76.41 & 83.39 & \textbf{72.00} & \textbf{68.75} & 67.52          & 68.54           & 77.52         \\
\emph{Graphene Smatch}      & \textbf{76.32}  & \textbf{80.26}        & \textbf{76.86} & \textbf{83.62} & 70.60 & 68.39 & 67.93          & \textbf{68.93}           & \textbf{77.79}         \\
\hline \hline
\textbf{Little Prince}  &  &   &   &  &   &  &  &  &  \\
SPRING          & 77.85           & 82.31          & 78.85          & 84.68          & 60.53          & 70.72          & 60.53          & \textbf{68.28}  & 77.78          \\

APT        & 75.21           & 80.07          & 76.12          & 85.29          & 65.15          & 67.92          & 69.70 & 63.28           & 75.31          \\
T5              & 77.66            & 81.99 & 78.53          & 85.12          &   58.06        & 72.33          & 59.35          & 67.03           & 78.30 \\
Cai\&Lam              & 71.03            & 75.91 & 72.07          & 80.18          & 22.73          & 57.51          & 31.50          & 59.29           & 72.02 \\\hline
\emph{Graphene 4S}       & 77.91  & 82.40          & 78.86 & 84.91 & 61.54 & 73.58 & 60.65          & 64.77           & 78.12          \\
\emph{Graphene Support}      & 78.54  & 82.81          & 79.44 & 85.52 & \textbf{64.05} & 75.11 & 63.45          & 67.83           & 78.72          \\
\emph{Barzdins}      & 79.21  & 83.47          & 80.12 & 85.74 & 59.21 & 74.06 & 61.84          & \textbf{70.45}           & 79.35       \\
\emph{Graphene Smatch}      & \textbf{79.52}  & \textbf{83.68}          & \textbf{80.44} & \textbf{86.10} & 63.16 & \textbf{75.11} & \textbf{63.89}          & 70.19           & \textbf{79.66} 
\\
\hline
\end{tabular}
\end{table}

\begin{figure*}[htb]
\center{\includegraphics[width=1.0\textwidth]{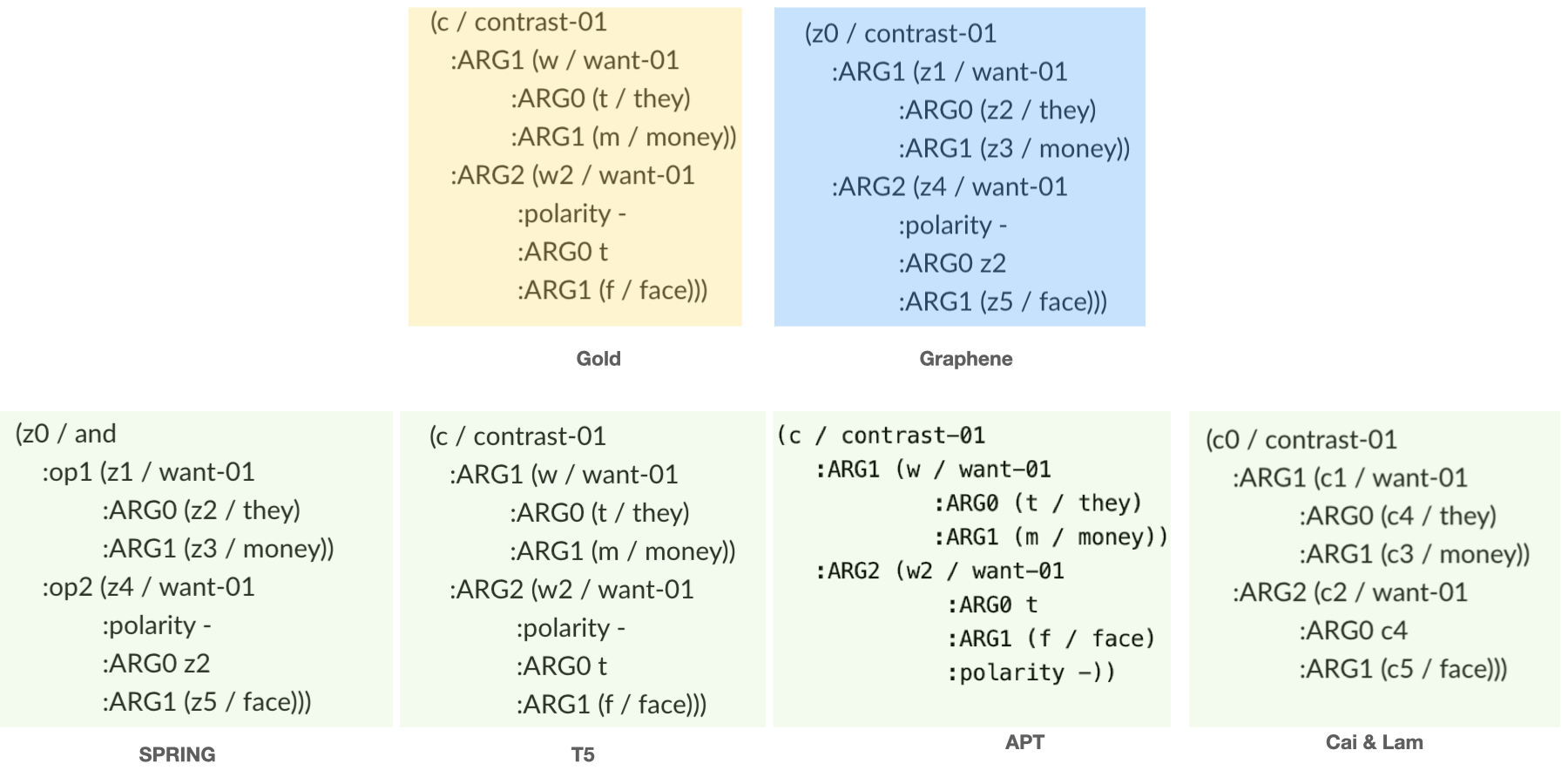}}
\caption{\label{fig:ensemble works} The gold AMR and the ensemble AMR graph of SPRING, T5, APT and Cai\&Lam using the Graphene algorithm for the sentence \emph{``They want money, not the face"}. }
\end{figure*}

\subsection{How the ensemble algorithm works}
We explore a few examples to demonstrate the reason why the ensemble method works. Figure \ref{fig:ensemble works} shows a sentence with a gold AMR in Penman format and a list of AMRs corresponding to the prediction of SPRING \citep{spring}, T5  \citep{t5}, APT \citep{zhou2021apt}, Cai and Lam \citep{cai-lam-2020-amr} parser and the ensemble graph given by  Graphene.

In this particular example, with the sentence \emph{``They want money, not the face"}, the AMR prediction from SPRING is inaccurate. Graphene corrects the prediction thanks to the votes given from the other models. In particular, the label $and$ of the root node $z_0$ of SPRING prediction was corrected to $contrast-01$ because T5, APT and Cai\&Lam parsers all vote for $contrast-01$. On the other hand, the labels $:op1$ and $:op2$  of the edges $(z_0, z_1)$  and $(z_0, z_4)$ were modified to have the correct labels $:ARG1$ and $:ARG2$ accordingly thanks to the votes from the other models. We can also see that even though the Cai\&Lam method misses polarity prediction, since the other models predict polarity correctly, the ensemble prediction does not inherit this mistake. Putting everything together, the prediction from Graphene perfectly matches with the gold AMR graph in this example.

\begin{table}[htp]
\centering
\caption{\label{tab:support}
The average total support and Smatch score of the prediction from SPRING, Graphene with SPRING as a pivot and Graphene considering every prediction as a pivot respectively. The support is highly correlated with  Smatch score.}
\small

\begin{tabular}{|l|rr|rr|rr|rr|rr|}
\hline        
              & \multicolumn{2}{c|}{\textbf{AMR 2.0}}
              & \multicolumn{2}{c|}{\textbf{AMR 3.0}} 
              & \multicolumn{2}{c|}{\textbf{BIO}}
              & \multicolumn{2}{c|}{\textbf{LP}}
              & \multicolumn{2}{c|}{\textbf{New3}}\\
  & Sup.       & Smat. & Sup.       & Smat.               & Sup.   & Smat.  & Sup.   & Smat.  & Sup.   & Smat.          \\ \hline
SPRING       & 170.15                    & 84.08 & 136.90                   & 83.14                & 166.86       & 60.52 & 69.33 & 77.85 & 118.27 & 74.66  \\
SPR. pivot       & 172.70                    & 84.70 & 139.42                    & 83.73                & 169.97       & 61.56 & 70.97 & 78.22 & 120.85& 74.83  \\
Graphene          & \textbf{175.73}                    & \textbf{85.85} & \textbf{142.07}                    & \textbf{84.43}                & \textbf{179.38}  & \textbf{62.29} & \textbf{72.64} & \textbf{78.54} & \textbf{123.62}      & \textbf{75.60}  \\\hline
\end{tabular}
\end{table}

%\begin{table}[htp]
%\centering
%\caption{\label{tab:support} \textcolor{red}{ old results}
%The average total support and Smatch score of the prediction from SPRING, Graphene with SPRING as a pivot and Graphene considering every prediction as a pivot respectively. The support is highly correlated with  Smatch score.}
%\begin{tabular}{|l|rr|rr|rr|rr|rr|}
%\hline        
%              & \multicolumn{2}{c|}{\textbf{AMR 2.0}}
%              & \multicolumn{2}{c|}{\textbf{AMR 3.0}} 
%              & \multicolumn{2}{c|}{\textbf{BIO}}
%              & \multicolumn{2}{c|}{\textbf{LP}}
%              & \multicolumn{2}{c|}{\textbf{New3}}\\
%  & Sup.       & Smatch & Sup.       & Smatch               & Sup.   & Smatch  & Sup.   & Smatch  & Sup.   & Smatch          \\ \hline
%SPRING       & 170.15                    & 84.08 & 136.90                   & 83.60                & 166.86       & 60.52 & 69.33 & 79.60 & 118.27 & 74.64  \\
%SPRING pivot       & 172.70                    & 84.70 & 139.42                    & 83.63                & 169.97       & 61.57 & 70.97 & 80.03 & 120.85& 74.79  \\
%Graphene          & \textbf{175.73}                    & \textbf{85.59} & \textbf{142.07}                    & \textbf{84.43}                & \textbf{179.38}  & \textbf{62.27} & \textbf{72.64} & \textbf{80.48} & \textbf{123.62}      & \textbf{75.57}  \\\hline
%\end{tabular}
%\end{table}

The Graphene algorithm searches for the graph that has the largest support from all individual graphs. One question that arises from this is whether the support is correlated with the accuracy of AMR parsing. Table \ref{tab:support} shows the support and the Smatch score of three models in the standard benchmark datasets. The first model is  SPRING, while the second one denoted as \emph{Graphene SPRING pivot} starts with a SPRING prediction as a pivot and corrects the prediction using votes from other models. The last model corresponds to the Graphene algorithm that polls the votes while considering every prediction as a pivot for correction and selecting the best one. Since Graphene looks for the best pivot to have better-supported ensemble graphs, the total supports of the Graphene predictions are larger than the Graphene SPRING pivot predictions. From the table, we can also see that the total support is highly correlated to the Smatch score. Namely, Graphene has higher support in all the benchmark datasets and a higher Smatch score than Graphene SPRING pivot. This experiment suggests that by optimizing the total support we can obtain the ensemble graphs with higher Smatch score.

\subsection{Comparison to \cite{barzdins}}
\label{subsec:comparison}
\cite{barzdins}\footnote{We would like to thank Juri Opitz for pointing us to missing references as a very useful feedback for our published work.} proposed a character-level based neural method for parsing texts into AMRs. In order to improve the robustness of the neural parser, the authors proposed an ensemble technique which selects among the prediction graphs the one that has the highest average Smatch when it is compared against the other predictions. The key difference between Barzdins' approach and our approach is that while our solution first modifies the predictions to create new prediction candidates for ensemble prediction, Barzdins' approach directly selects a prediction among existing predictions. 

We provided experimental results to compare to \cite{barzdins}  in Table \ref{tab:amr} and Table \ref{tab:out-of-distribution} to \cite{barzdins}. Recall that our Graphene algorithm has two stages. In the first one, it modifies all the predictions to make new pivot graphs that are most supported by the other graphs. In the second stage it chooses the best candidate among the pivot and original predictions with respect to the average support. The second stage is very similar to \cite{barzdins}, the only difference is in the objective function where support was used instead of Smatch. For general graph prediction problems when the evaluation metrics is not known in advance Graphene can rely on the support. However, for AMR parsing problem where Smatch is the evaluation metrics, Graphene can also use Smatch as the criteria to select the best candidate similar to \cite{barzdins} does, we denote this version of Graphene as Graphene\_Smatch, the other version that uses support is denoted as  Graphene\_Support. 

As we can see in both Table \ref{tab:amr} and Table \ref{tab:out-of-distribution}, Bazdins's method works very well for 4 out of 5 datasets. Since it targets optimizing the Smatch, it provides better predictions than Graphene\_Support  except for the BIO dataset. However, when Graphene uses Smatch to select the best candidate from the pivoting and the original predictions, Graphene\_Smatch provides slightly better results in AMR 2.0, AMR 3.0 and LP, and outperform Bazdins's method in BIO and New3 datasets. 

\paragraph{Discussion} Our hypothesis is that if the predictions of each individual model are similar to each other and accurate, Bazdins methods and Graphene\_Smatch are comparable. This can be clearly observed in AMR 2.0, AMR 3.0 and LP. But when the predictions from the models are not accurate and differ, the modification step in Graphene helps to correct the predictions and thus it provides much better results in BIO and New3 dataset. Although both algorithms search for a graph that is most similar to all the prediction, Graphene search space is extended to pivoting graphs besides existing graphs so it provides better results.

\section{Related work}
\label{sec:related work}

\paragraph{Ensemble learning.} Ensemble learning is a popular machine learning approach that combines predictions from different learners to make a more robust and more accurate prediction. Many ensembling approaches have been proposed, such as bagging \citep{bagging} or boosting \citep{boosting}, the winning solutions in many machine learning competitions \citep{xgboost}. These methods are proposed mainly for regression or classification problems. Recently, structure output prediction emerges as an important problem thanks to the advances in deep learning research. To apply popular ensembling techniques such as bagging or boosting, it is important to study the ensemble method for combining structure predictions. 

\paragraph{Ensemble structure prediction.} 

Previous studies have explored various ensemble learning approaches for dependency and constituent parsing: \citep{sagae-lavie-2006-parser} proposes a reparsing framework that takes the output from different parsers and maximizes the number of votes for a well-formed dependency or constituent structure; 
\citep{kuncoro-etal-2016-distilling} uses minimum Bayes risk inference to build a consensus dependency parser from an ensemble of independently trained greedy LSTM transition-based parsers with different random initializations. Note that a syntactic tree is a special graph structure in which nodes for a sentence from different parsers are roughly the same. In contrast, we propose an approach to ensemble graph predictions in which both graph nodes and edges can be different among base predictions. 

\paragraph{Ensemble methods for AMR parsing.} Parsing text to AMR is an important research problem. State-of-the-art approaches in AMR parsing are divided into three categories. Sequence to sequence models \citep{spring,konstas2017neural, van2017neural,xu2020improving} consider the AMR parsing as a machine translation problem that translates texts to AMR graphs. The transition-based methods \citep{zhou2021apt} predicts a sequence of actions given the input text, and then the action sequence is turned into an AMR graph using an oracle decoder. Lastly, graph-based methods \citep{cai-lam-2020-amr} directly construct the AMR graphs from textual data. All these methods are complementary to each other and thus ensemble methods can leverage the strength of these methods to create a better prediction, as demonstrated in this paper. Ensemble of AMR predictions from a single type of model is studied in \citep{zhou2021apt} where the authors demonstrated that by combining predictions from three different model's checkpoints they gain %accuracy
performance 
improvement in the final prediction. However, ensemble in sequential decoding requires that all predictions are from the same type of models. It is not applicable for cases when the predictions are from different types of models such as seq2seq, transition-based or graph-based models.  In contrast to that approach, our algorithm is model-agnostic, i.e. it can combine predictions from different models. In our experiments, we have demonstrated the benefit of combining predictions from different models, with additional gains in performance 
%accuracy 
compared to the ensemble of predictions from a single model's checkpoints.  

\paragraph{Comparison to Bazdins et al.} \cite{barzdins} proposed a character-level based neural method for parsing texts into AMRs. In order to improve the robustness of the neural parser, the authors proposed an ensemble technique which selects among the prediction graphs the one that has the highest average SMATCH when it is compared against the other predictions. The key difference between Barzdins' approach and our approach is that while our solution modifies the predictions to create new prediction candidates for ensemble prediction, Barzdins' approach only selects a prediction among existing predictions. In order to demonstrate how creating new prediction candidates from existing predictions before selecting the best candidates help achieving a better prediction we have discussed the new results added to Table \ref{tab:amr} and \ref{tab:out-of-distribution}.

\section{Conclusions and future work}
\label{sec:conclusions and future work}
In this paper, we formulate the graph ensemble problem, study its computational intractability, and provide an algorithm for constructing graph ensemble predictions. We validate our approach with AMR parsing problems. The experimental results show that the proposed approach outperforms the previous state-of-the-art AMR parsers and achieves new state-of-the-art results in five different benchmark datasets. We demonstrate that the proposed ensemble algorithm not only works well for in-distribution but also for out-of-distribution evaluations. This result has a significant practical impact, especially when applying the proposed method to domain-specific texts where training data is not available. Moreover, our approach is model-agnostic, which means that it can be used to combine predictions from different models without the requirement of having an access to model training. In general,  our approach provides the basis for graph ensemble, studying classical ensemble techniques such as bagging, boosting, or stacking for graph ensemble is a promising future research direction that is worth considering to improve the results further.

% \clearpage
% \normal
\bibliographystyle{plainnat}
% \bibliography{references}
% \bibliographystyle{aistats2020}
%\input{refs}
% \newpage
\bibliography{reference}

\begin{thebibliography}{29}
\providecommand{\natexlab}[1]{#1}
\providecommand{\url}[1]{\texttt{#1}}
\expandafter\ifx\csname urlstyle\endcsname\relax
  \providecommand{\doi}[1]{doi: #1}\else
  \providecommand{\doi}{doi: \begingroup \urlstyle{rm}\Url}\fi

\bibitem[amr()]{amrlib}
{AMRLIB} data preprocessing
  https://github.com/bjascob/amrlib/tree/master/scripts.

\bibitem[dat()]{data}
{AMR} benchmark datasets. https://amr.isi.edu/download.html.

\bibitem[goo()]{google}
{English Web Treebank} dataset: https://catalog.ldc.upenn.edu/ldc2012t13.
\newblock URL \url{https://catalog.ldc.upenn.edu/LDC2012T13}.

\bibitem[Bahiense et~al.(2012)Bahiense, Mani{\'c}, Piva, and
  De~Souza]{npcomplete}
Laura Bahiense, Gordana Mani{\'c}, Breno Piva, and Cid~C De~Souza.
\newblock The maximum common edge subgraph problem: A polyhedral investigation.
\newblock \emph{Discrete Applied Mathematics}, 160\penalty0 (18):\penalty0
  2523--2541, 2012.

\bibitem[Banarescu et~al.(2013)Banarescu, Bonial, Cai, Georgescu, Griffitt,
  Hermjakob, Knight, Koehn, Palmer, and Schneider]{banarescu2013abstract}
Laura Banarescu, Claire Bonial, Shu Cai, Madalina Georgescu, Kira Griffitt, Ulf
  Hermjakob, Kevin Knight, Philipp Koehn, Martha Palmer, and Nathan Schneider.
\newblock Abstract meaning representation for sembanking.
\newblock In \emph{Proceedings of the 7th Linguistic Annotation Workshop and
  Interoperability with Discourse}, pages 178--186, 2013.

\bibitem[Barzdins and Gosko(2016)]{barzdins}
Guntis Barzdins and Didzis Gosko.
\newblock Riga at semeval-2016 task 8: Impact of smatch extensions and
  character-level neural translation on amr parsing accuracy.
\newblock \emph{arXiv preprint arXiv:1604.01278}, 2016.

\bibitem[Bevilacqua et~al.(2021)Bevilacqua, Blloshmi, and Navigli]{spring}
Michele Bevilacqua, Rexhina Blloshmi, and Roberto Navigli.
\newblock One {SPRING} to rule them both: {S}ymmetric {AMR} semantic parsing
  and generation without a complex pipeline.
\newblock In \emph{Proceedings of AAAI}, 2021.

\bibitem[Breiman(1996)]{bagging}
Leo Breiman.
\newblock Bagging predictors.
\newblock \emph{Machine learning}, 24\penalty0 (2):\penalty0 123--140, 1996.

\bibitem[Cai and Lam(2020{\natexlab{a}})]{cai-lam}
Deng Cai and Wai Lam.
\newblock {AMR} parsing via graph-sequence iterative inference.
\newblock In \emph{Proceedings of the 58th Annual Meeting of the Association
  for Computational Linguistics}, pages 1290--1301, Online, July
  2020{\natexlab{a}}. Association for Computational Linguistics.
\newblock \doi{10.18653/v1/2020.acl-main.119}.
\newblock URL \url{https://www.aclweb.org/anthology/2020.acl-main.119}.

\bibitem[Cai and Lam(2020{\natexlab{b}})]{cai-lam-2020-amr}
Deng Cai and Wai Lam.
\newblock {AMR} parsing via graph-sequence iterative inference.
\newblock In \emph{Proceedings of the 58th Annual Meeting of the Association
  for Computational Linguistics}, pages 1290--1301, Online, July
  2020{\natexlab{b}}. Association for Computational Linguistics.
\newblock \doi{10.18653/v1/2020.acl-main.119}.
\newblock URL \url{https://www.aclweb.org/anthology/2020.acl-main.119}.

\bibitem[Cai and Knight(2013)]{cai-knight-2013-smatch}
Shu Cai and Kevin Knight.
\newblock {S}match: an evaluation metric for semantic feature structures.
\newblock In \emph{Proceedings of the 51st Annual Meeting of the Association
  for Computational Linguistics (Volume 2: Short Papers)}, pages 748--752,
  Sofia, Bulgaria, August 2013. Association for Computational Linguistics.
\newblock URL \url{https://www.aclweb.org/anthology/P13-2131}.

\bibitem[Chen and Guestrin(2016)]{xgboost}
Tianqi Chen and Carlos Guestrin.
\newblock {XGB}oost: A scalable tree boosting system.
\newblock In \emph{Proceedings of the 22nd ACM SIGKDD International Conference
  on Knowledge Discovery and Data Mining}, 2016.

\bibitem[Damonte et~al.(2017)Damonte, Cohen, and
  Satta]{damonte-etal-2017-incremental}
Marco Damonte, Shay~B. Cohen, and Giorgio Satta.
\newblock An incremental parser for {A}bstract {M}eaning {R}epresentation.
\newblock In \emph{Proceedings of the 15th Conference of the {E}uropean Chapter
  of the Association for Computational Linguistics: Volume 1, Long Papers},
  pages 536--546, Valencia, Spain, April 2017. Association for Computational
  Linguistics.
\newblock URL \url{https://www.aclweb.org/anthology/E17-1051}.

\bibitem[Domingos(2000)]{domingos2000unified}
Pedro Domingos.
\newblock A unified bias-variance decomposition.
\newblock In \emph{Proceedings of 17th International Conference on Machine
  Learning}, pages 231--238, 2000.

\bibitem[Dong et~al.(2020)Dong, Yu, Cao, Shi, and Ma]{dong2020survey}
Xibin Dong, Zhiwen Yu, Wenming Cao, Yifan Shi, and Qianli Ma.
\newblock A survey on ensemble learning.
\newblock \emph{Frontiers of Computer Science}, 14\penalty0 (2):\penalty0
  241--258, 2020.

\bibitem[Kapanipathi et~al.(2020)Kapanipathi, Abdelaziz, Ravishankar, Roukos,
  Gray, Astudillo, Chang, Cornelio, Dana, Fokoue,
  et~al.]{kapanipathi2020question}
Pavan Kapanipathi, Ibrahim Abdelaziz, Srinivas Ravishankar, Salim Roukos,
  Alexander Gray, Ramon Astudillo, Maria Chang, Cristina Cornelio, Saswati
  Dana, Achille Fokoue, et~al.
\newblock Question answering over knowledge bases by leveraging semantic
  parsing and neuro-symbolic reasoning.
\newblock \emph{arXiv preprint arXiv:2012.01707}, 2020.

\bibitem[Konstas et~al.(2017)Konstas, Iyer, Yatskar, Choi, and
  Zettlemoyer]{konstas2017neural}
Ioannis Konstas, Srinivasan Iyer, Mark Yatskar, Yejin Choi, and Luke
  Zettlemoyer.
\newblock Neural {AMR}: Sequence-to-sequence models for parsing and generation.
\newblock \emph{ACL}, 2017.

\bibitem[Kuncoro et~al.(2016)Kuncoro, Ballesteros, Kong, Dyer, and
  Smith]{kuncoro-etal-2016-distilling}
Adhiguna Kuncoro, Miguel Ballesteros, Lingpeng Kong, Chris Dyer, and Noah~A.
  Smith.
\newblock Distilling an ensemble of greedy dependency parsers into one {MST}
  parser.
\newblock In \emph{Proceedings of the 2016 Conference on Empirical Methods in
  Natural Language Processing}, pages 1744--1753, Austin, Texas, November 2016.
  Association for Computational Linguistics.
\newblock \doi{10.18653/v1/D16-1180}.
\newblock URL \url{https://www.aclweb.org/anthology/D16-1180}.

\bibitem[Lee et~al.(2020)Lee, Astudillo, Naseem, Reddy, Florian, and
  Roukos]{youngsuk2020silver}
Young-Suk Lee, Ramon~Fernandez Astudillo, Tahira Naseem, Revanth~Gangi Reddy,
  Radu Florian, and Salim Roukos.
\newblock Pushing the limits of {AMR} parsing with self-learning.
\newblock In \emph{Findings of EMNLP 2020}, pages 3208--3214, 2020.

\bibitem[Li et~al.(2020)Li, Min, Iyer, Mehdad, and Yih]{blink}
Belinda~Z. Li, Sewon Min, Srinivasan Iyer, Yashar Mehdad, and Wen-tau Yih.
\newblock Efficient one-pass end-to-end entity linking for questions.
\newblock In \emph{EMNLP}, 2020.

\bibitem[Lim et~al.(2020)Lim, Oh, Jang, Yang, and Lim]{lim-etal-2020-know}
Jungwoo Lim, Dongsuk Oh, Yoonna Jang, Kisu Yang, and Heuiseok Lim.
\newblock {I} know what you asked: Graph path learning using {AMR} for
  commonsense reasoning.
\newblock In \emph{Proceedings of the 28th International Conference on
  Computational Linguistics}, pages 2459--2471, Barcelona, Spain (Online),
  December 2020. International Committee on Computational Linguistics.
\newblock \doi{10.18653/v1/2020.coling-main.222}.
\newblock URL \url{https://www.aclweb.org/anthology/2020.coling-main.222}.

\bibitem[Raffel et~al.(2020)Raffel, Shazeer, Roberts, Lee, Narang, Matena,
  Zhou, Li, and Liu]{t5}
Colin Raffel, Noam Shazeer, Adam Roberts, Katherine Lee, Sharan Narang, Michael
  Matena, Yanqi Zhou, Wei Li, and Peter~J. Liu.
\newblock Exploring the limits of transfer learning with a unified text-to-text
  transformer.
\newblock \emph{Journal of Machine Learning Research}, 21\penalty0
  (140):\penalty0 1--67, 2020.
\newblock URL \url{http://jmlr.org/papers/v21/20-074.html}.

\bibitem[Sagae and Lavie(2006)]{sagae-lavie-2006-parser}
Kenji Sagae and Alon Lavie.
\newblock Parser combination by reparsing.
\newblock In \emph{Proceedings of the Human Language Technology Conference of
  the {NAACL}, Companion Volume: Short Papers}, pages 129--132, New York City,
  USA, June 2006. Association for Computational Linguistics.
\newblock URL \url{https://www.aclweb.org/anthology/N06-2033}.

\bibitem[Schapire and Freund(2013)]{boosting}
Robert~E Schapire and Yoav Freund.
\newblock Boosting: Foundations and algorithms.
\newblock \emph{Kybernetes}, 2013.

\bibitem[Valentini and Dietterich(2004)]{valentini2004bias}
Giorgio Valentini and Thomas~G Dietterich.
\newblock Bias-variance analysis of support vector machines for the development
  of svm-based ensemble methods.
\newblock \emph{Journal of Machine Learning Research}, 5\penalty0
  (Jul):\penalty0 725--775, 2004.

\bibitem[Van~Noord and Bos(2017)]{van2017neural}
Rik Van~Noord and Johan Bos.
\newblock Neural semantic parsing by character-based translation: Experiments
  with abstract meaning representations.
\newblock \emph{arXiv preprint arXiv:1705.09980}, 2017.

\bibitem[Vaswani et~al.(2017)Vaswani, Shazeer, Parmar, Uszkoreit, Jones, Gomez,
  Kaiser, and Polosukhin]{vaswani2017transformer}
Ashish Vaswani, Noam Shazeer, Niki Parmar, Jakob Uszkoreit, Llion Jones,
  Aidan~N Gomez, Lukasz Kaiser, and Illia Polosukhin.
\newblock Attention is all you need.
\newblock \emph{Advances in Neural Information Processing Systems}, pages
  5998--6008, 2017.

\bibitem[Xu et~al.(2020)Xu, Li, Zhu, Zhang, and Zhou]{xu2020improving}
Dongqin Xu, Junhui Li, Muhua Zhu, Min Zhang, and Guodong Zhou.
\newblock Improving {AMR} parsing with sequence-to-sequence pre-training.
\newblock \emph{EMNLP}, 2020.

\bibitem[Zhou et~al.(2021)Zhou, Naseem, Astudillo, and Florian]{zhou2021apt}
Jiawei Zhou, Tahira Naseem, Ramon~Fernandez Astudillo, and Radu Florian.
\newblock {AMR} parsing with action-pointer transformer.
\newblock In \emph{Proceedings of the 2021 Conference of the North American
  Chapter of the Association for Computational Linguistics: Human Language
  Technologies (NAACL 2021)}, pages 5585--5598, 2021.

\end{thebibliography}
% \clearpage

% \onecolumn 
% \normal
\appendix

\newpage
  \vbox{%
    \hsize\textwidth
    \linewidth\hsize
    \vskip 0.1in
  \hrule height 4pt
  \vskip 0.25in
  \vskip -5.5pt%
  \centering
    {\LARGE\bf{Ensemble Graph Prediction \\
    Supplementary Material} \par}
      \vskip 0.29in
  \vskip -5.5pt
  \hrule height 1pt
  \vskip 0.09in%
    
  \vskip 0.2in
  }
\section{Running time}
Figure \ref{fig:time} shows the average running time of the Graphene algorithm. The horizontal axis corresponds to the average graph size (the number of triples) while the vertical axis shows the average running time (in seconds). We can see that the running time depends on the average size of the AMR graphs. Since AMR graph size is proportional to the input sentence length, the largest average graph has around 50 triples.  Graphene requires less than 2 seconds on an 8-core CPU machine to make an ensemble from 7 models. 
\begin{figure*}[htb]
\center{\includegraphics[width=\textwidth]{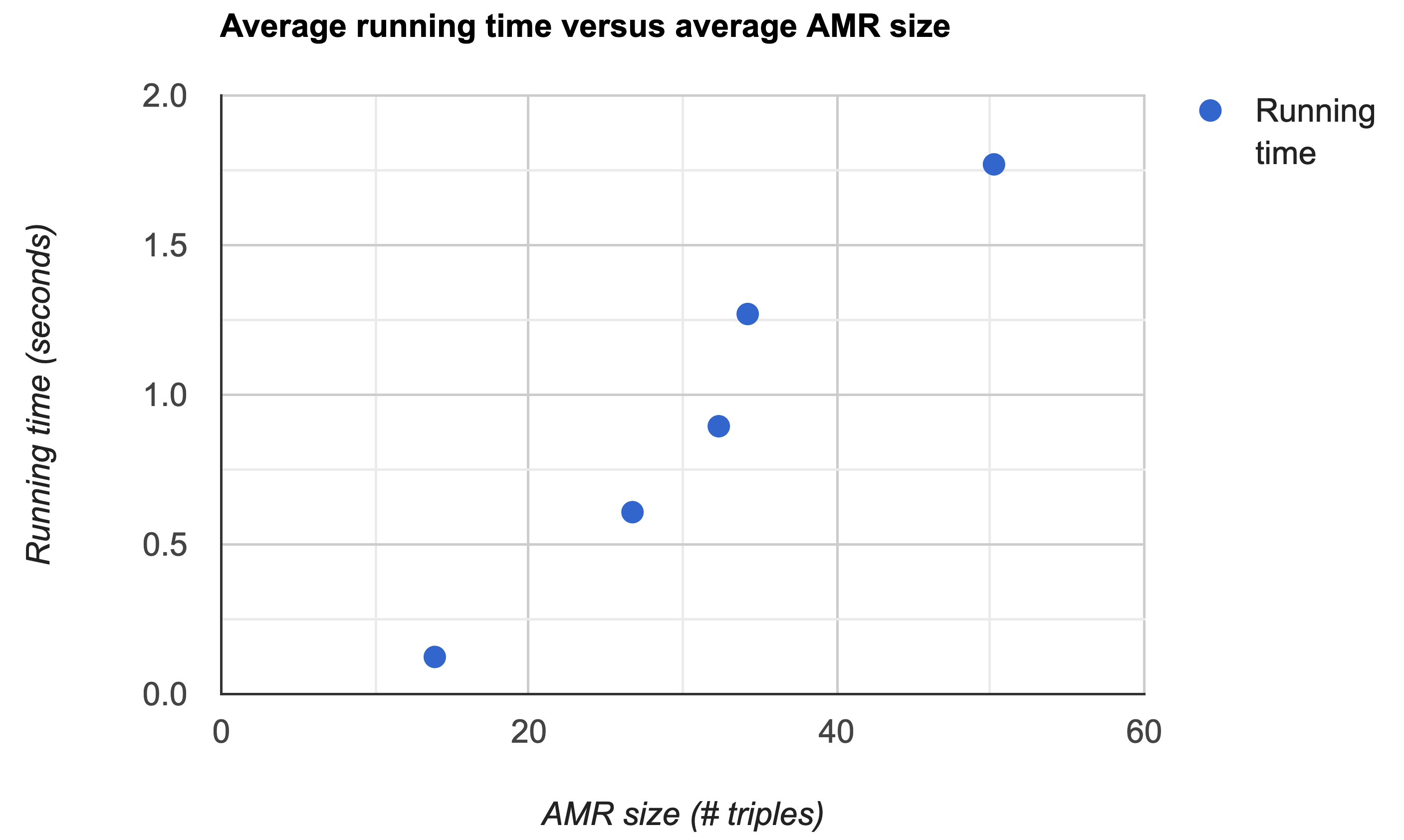}}
\caption{\label{fig:time} Average running time of the Graphene algorithm versus the average graph size in the LP, New3, AMR 3.0, AMR 2.0, and BIO datasets respectively. }
\end{figure*}

\section{On the support threshold}
The popular VotingClassifier algorithm implemented in scikit-learn \footnote{\url{https://scikit-learn.org/stable/modules/generated/sklearn.ensemble.VotingClassifier.html}}  follows the majority vote rule where the label with the most votes is selected as the final prediction. Therefore, we apply the same rule in our experimental settings where setting theta = 0.5 is equivalent to the majority vote rule in classification problems.

If there is an independent validation set, this hyper-parameter can be tuned to choose the right theta value for that dataset. For example, in the AMR 2.0 dataset, the results of ensembling 4 Spring models, APT model, and T5 models on the validation set (the dev split) when theta is varied are reported in Table \ref{tab:theta}.
 
\begin{table}[h]
\centering
\caption{\label{tab:theta}  The results of ensembling 4 Spring models, APT model, and T5 models on the validation set (the dev split) when $\theta$ is varied. On this dev set, $\theta = 0.5$ is a proper choice for AMR 2.0.}
\begin{tabular}{l|lllll}
\hline
\textbf{} & \textbf{$\theta=0.1$} & \textbf{$\theta=0.3$} & \textbf{$\theta=0.5$} & \textbf{$\theta=0.7$} & \textbf{$\theta=0.9$} \\ \hline
Smatch         & 81.64             & 85.01             & \textbf{85.49}         & 85.12                        & 83.68           \\
Precision         & 76.13             & 83.62             & 85.86         & 86.53                        & \textbf{86.54}           \\
Recall         & \textbf{88.00}             & 86.55             & 85.13         & 83.75                        & 81.01           \\ \hline
\end{tabular}
\end{table}

%\begin{table}[h]
%\centering
%\caption{\label{tab:theta} \textcolor{red}{ Old results} The results of ensembling 4 Spring models, APT model, and T5 models on the validation set (the dev split) when $\theta$ is varied. On this dev set, $\theta = 0.5$ is a proper choice for AMR 2.0.}
%\begin{tabular}{l|lllll}
%\hline
%\textbf{} & \textbf{$\theta=0.1$} & \textbf{$\theta=0.3$} & \textbf{$\theta=0.5$} & \textbf{$\theta=0.7$} & \textbf{$\theta=0.9$} \\ \hline
%Smatch         & 0.806             & 0.854             & \textbf{0.857}         & 0.848                        & 0.831           \\
%Precision         & 0.738             & 0.826             & 0.854         & 0.866                        & 0.864           \\
%Recall         & 0.887             & 0.873             & 0.861         & 0.830                        & 0.801           \\ \hline
%\end{tabular}
%\end{table}

Based on this result on an independent dev set, theta=0.5 is the right choice for AMR 2.0. Note that setting theta is a trade-off between precision and recall.

\section{Comparison with median baselines}
Beside the Graphene 4S baseline, we provide the results in Table \ref{tab:median} of the following baseline approaches:
\begin{itemize}
    \item  Uniform sampling: for each set of predictions we sample the graph uniformly at random, this approach is equivalent to the “median” representative from a set.
    \item Ideal median: assumes that the gold AMRs are available for the test set (hence named as "ideal"). We computed the Smatch of each prediction with the gold AMR and use the AMR with the median Smatch score as the final prediction.
\end{itemize}

\begin{table}[h]
\centering
\caption{\label{tab:median} Comparison with median baselines }
\begin{tabular}{l|lllll}
\hline
\textbf{} & \textbf{AMR 2.0} & \textbf{AMR 3.0} & \textbf{BIO} & \textbf{new3} & \textbf{LP} \\ \hline
uniform Sampling         & 82.58             & 82.98             & 56.00         & 71.18                        & 76.17           \\
Ideal median         & 83.80             & 83.66             & 57.72         & 73.06                        & 77.14           \\
Graphene        & \textbf{85.85}             & \textbf{84.41}             & \textbf{62.29}         & \textbf{75.60}                        & \textbf{78.54}           \\ \hline
\end{tabular}
\end{table}

%\begin{table}[h]
%\centering
%\caption{\label{tab:median} \textcolor{red}{Old results} Comparison with median baselines }
%\begin{tabular}{l|lllll}
%\hline
%\textbf{} & \textbf{AMR 2.0} & \textbf{AMR 3.0} & \textbf{BIO} & \textbf{new3} & \textbf{LP} \\ \hline
%uniform Sampling         & 0.824             & 0.828             & 0.560         & 0.711                        & 0.779           \\
%Ideal median         & 0.835             & 0.836             & 0.577         & 0.731                        & 0.788           \\
%Graphene        & \textbf{0.858}             & \textbf{0.844}             & \textbf{0.622}         & \textbf{0.757}                        & \textbf{0.804}           \\ \hline
%\end{tabular}
%\end{table}

\section{Pivot selection}
\begin{figure*}[htb]
\center{\includegraphics[width=\textwidth]{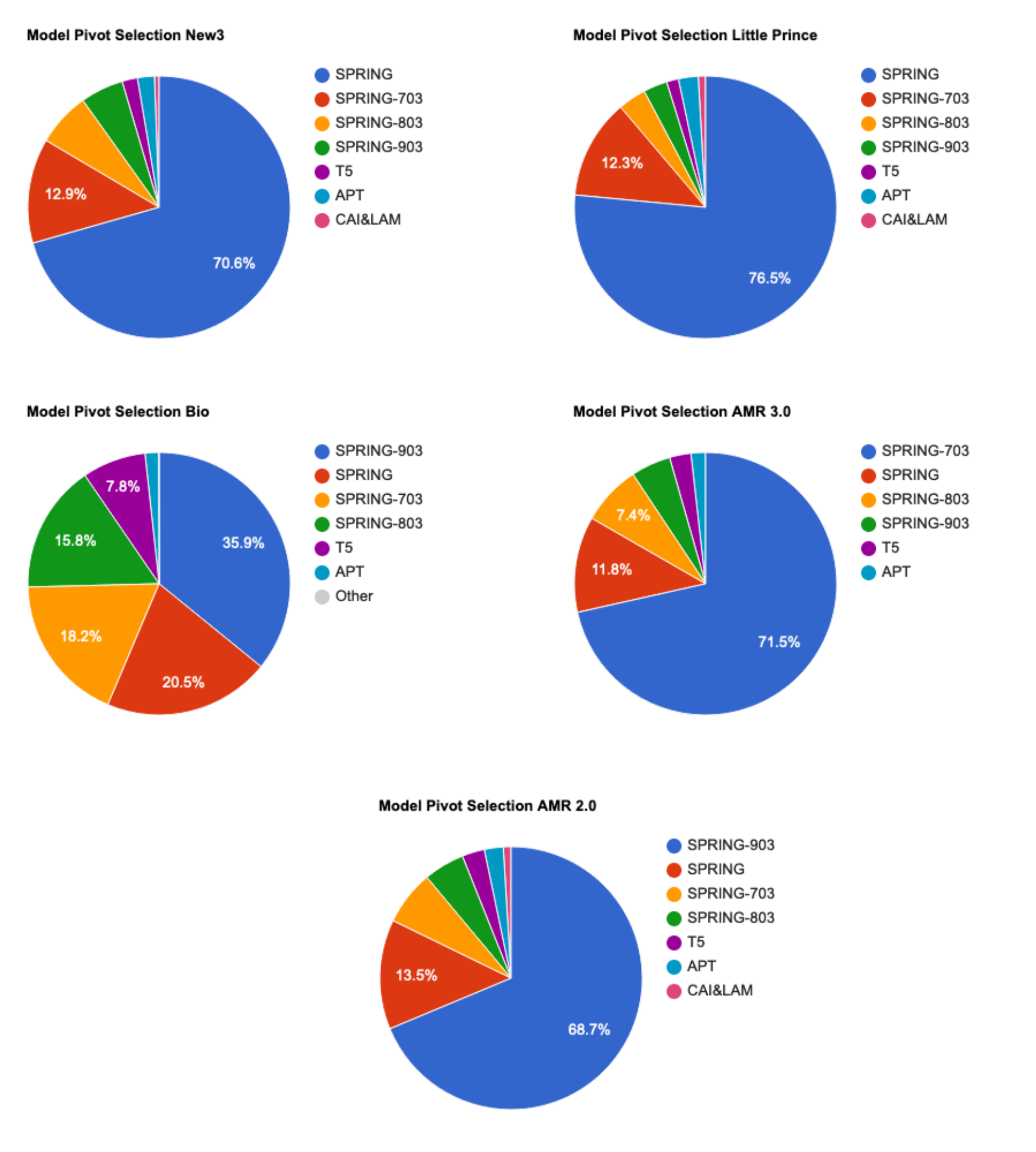}}
\caption{\label{fig:pivot} Pie charts shows the percentage of the number of times each model was selected as the best pivot in the Graphene algorithm. Notice that when tight happens, the ensemble created when  first algorithm is considered as the pivot is selected as the final ensemble graph. }
\end{figure*}
Figure \ref{fig:pivot} shows the pie-charts with the percentage summarizing the number of times that the prediction created when each algorithm is chosen as a pivot graph is selected as the final prediction. Notice that the order of the algorithms matters because when tight happens, the ensemble is chosen from the first algorithm in the list.

The results show that all algorithms contribute to the final predictions. In the Bio dataset where the test data is from a specific domain that differs from the training data domain, Graphene benefits from the model diversity when it leverages predictions from all models effectively.

\section{Robustness on down-sampled training data}

\begin{table}[htp]
\centering
\caption{\label{tab:robustness} When the training data is down-sampled, the gain of using Graphene is enlarged.}
\begin{tabular}{|l|llll|llll|}
\hline        
              & \multicolumn{4}{c|}{\textbf{Sample rate 0.6}}
              & \multicolumn{4}{c|}{\textbf{Sample rate 0.8}} \\
Methods/Data  & AMR 2.0       & BIO & New3 & LP & AMR 2.0       & BIO & New3 & LP  \\ \hline
SPRING 603       & 82.40 & 58.22 & 73.81 & 76.75 & 83.28 & 58.85 & 74.13 & 76.96  \\
SPRING 703       & 82.74 & 57.93 & 73.49 & 76.60 & 83.43 & 59.50 & 74.71 & 77.40  \\
SPRING 803       & 82.85 & 58.72 & 73.42 & 76.68 & 83.39 & 58.97 & 73.57 & 77.39  \\
SPRING 903       & 82.81 & 58.80 & 74.09 & 76.96 & 83.12 & 57.52 & 73.50 & 76.70  \\
T5       & 82.08 & 56.46 & 72.84 & 76.31 & 82.59 & 58.69 & 73.42 & 77.70  \\
Graphene       & \textbf{84.20} & \textbf{61.66} & \textbf{75.01} & \textbf{77.79} & \textbf{84.70} & \textbf{62.23} & \textbf{75.98} & \textbf{78.09}  \\\hline
\end{tabular}
\end{table}

%\begin{table}[htp]
%\centering
%\caption{\label{tab:robustness} \textcolor{red}{Old results} When the training data is down-sampled, the gain of using Graphene for ensembling the predictions is enlarged.}
%\begin{tabular}{|l|llll|llll|}
%\hline        
%              & \multicolumn{4}{c|}{\textbf{Sample rate 0.6}}
%              & \multicolumn{4}{c|}{\textbf{Sample rate 0.8}} \\
%Methods/Data  & AMR 2.0       & BIO & New3 & LP & AMR 2.0       & BIO & New3 & LP  \\ \hline
%SPRING 603       & 0.823 & 0.582 & 0.738 & 0.785 & 0.832 & 0.588 & 0.741 & 0.787  \\
%SPRING 703       & 0.827 & 0.579 & 0.734 & 0.783 & 0.834 & 0.595 & 0.746 & 0.791  \\
%SPRING 803       & 0.828 & 0.586 & 0.734 & 0.784 & 0.833 & 0.589 & 0.735 & 0.791  \\
%SPRING 903       & 0.827 & 0.587 & 0.741 & 0.787 & 0.831 & 0.575 & 0.734 & 0.785  \\
%T5       & 0.819 & 0.565 & 0.728 & 0.781 & 0.824 & 0.586 & 0.734 & 0.795  \\
%Graphene       & \textbf{0.843} & \textbf{0.617} & \textbf{0.750} & \textbf{0.796} & \textbf{0.848} & \textbf{0.622} & \textbf{0.760} & \textbf{0.799}  \\\hline
%\end{tabular}
%\end{table}

We down-sampled the AMR 2.0 training data with sample rates 0.6 and 0.8. Then 4 Spring models with different random seeds and the T5 model were trained on these two sample sets. The Smatch score on AMR 2 test sets and on the out-of-distribution sets (LP, New3, Bio) are reported in Table \ref{tab:robustness}.

Compared to the best individual models, Graphene is more robust and 1.35, 2.86, 0.92, and 0.83 points better when the sample rate is equal to 0.6. While compared to the best individual models, Graphene is more robust and 1.27, 2.73, 1.27 and 0.39 points better better when the sample rate is equal to 0.8. This result demonstrates that the proposed method is robust with respect to smaller training data.

\section{Ties broken arbitrarily}
When many ensemble graphs have the same support, Graphene chooses the ensemble graph created when the first model in the list is chosen as the pivot. Table \ref{tab:amr first} shows the results when each model is put first in the list. If we have a validation set like the case with AMR 2.0 or AMR 3.0, we can tune the right input order to achieve the best performance on the validation set. 

In case there is no validation set available, to mitigate the impact of random input order, we can break the ties arbitrarily, the results of ties broken arbitrarily are reported in Table \ref{tab:ties broken}.
\begin{table}[htbp]
\centering
\caption{\label{tab:amr first}  Results of Graphene when each model is put first in the list.
}
\begin{tabular}{l|lllllllll}
\hline
\textbf{Grapphene Models} & \textbf{S} & \textbf{S703} & \textbf{S803} & \textbf{S903} & \textbf{T5}    & \textbf{APT}  & \textbf{Cai\&Lam}   \\ \hline

 AMR 2.0        & 85.66           & 85.77          & 85.81          & \textbf{85.87}          & 85.65          & 85.67 & 85.11                \\
AMR 3.0        & \textbf{84.44}           & 84.42          & 84.34          & \textbf{84.44}          & 84.35          & 84.18          &  NA          \\
Bio             & 62.38           &  62.41         & 62.39          & \textbf{62.44}         &          62.38 & 62.41          & 62.34                   \\

LP & 78.65           & 78.63          & \textbf{78.75}          & 78.70         & 78.65          & 78.23          & 77.75                  \\
New3     & 75.65  & 75.69  & \textbf{75.88} & 75.82 & 75.78 & 75.48 & 74.92 \\

\hline
\end{tabular}
\end{table}

%\begin{table}[htbp]
%\centering
%\caption{\label{tab:amr first} \textcolor{red}{Old results} Results of Graphene when each model is put first in the list.
%}
%\begin{tabular}{l|lllllllll}
%\hline
%\textbf{Grapphene Models} & \textbf{S} & \textbf{S703} & \textbf{S803} & \textbf{S903} & \textbf{T5}    & \textbf{APT}  & \textbf{Cai\&Lam}   \\ \hline
%SPRING        & 83.69           & 86.92          & 84.18          & 89.60          & 91.38          & 74.89          & 80.75          & 73.07  & 82.27          \\
%SPRING 703        & 83.79           & 87.04          & 84.32          & 89.62          & 91.66          & 74.85          & 79.93 & 73.88           & 82.33          \\
%SPRING 803        & 83.91           & 87.03          & 84.42          & 89.79          & 91.31          & 74.05          & 79.72 & 73.98           & 82.21          \\
% AMR 2.0        & 85.55           & 85.65          & 85.72          & \textbf{85.74}          & 85.54          & 85.54 & 84.97                \\
%AMR 3.0        & \textbf{84.34}           & 84.32          & 84.24          & \textbf{84.34}          & 84.27          & 84.06          &  NA          \\
%Bio             & 62.37           &  62.39         & 62.39          & \textbf{62.46}         &          62.38 & 62.38          & 62.31                   \\
%LP & 80.50           & 80.50          & \textbf{80.57}          & 80.51         & 80.46          & 80.11          & 79.54                  \\

%New3     & 75.61  & 75.69  & \textbf{75.86} & 75.77 & 75.74 & 75.47 & 74.94 \\

%\hline
%\end{tabular}
%\end{table}

\begin{table}[htbp]
\centering
\caption{\label{tab:ties broken} Results of Graphene when ties are broken arbitrarily.
%in the in-distribution evaluation.
}
\begin{tabular}{l|lllllllll}
\hline
\textbf{Grapphene Models} & \textbf{AMR 2.0} & \textbf{AMR 3.0} & \textbf{Bio} & \textbf{LP} & \textbf{New3}  \\ \hline
 Graphene 4SATC        & 85.52           & 84.27          & 62.39          & 78.50          & 75.66 \\             

\hline
\end{tabular}
\end{table}

%\begin{table}[htbp]
%\centering
%\caption{\label{tab:ties broken} \textcolor{red}{Old results} Results of %Graphene when ties are rboken arbitrarily.
%in the in-distribution evaluation.
%}
%\begin{tabular}{l|lllllllll}
%\hline
%\textbf{Grapphene Models} & \textbf{AMR 2.0} & \textbf{AMR 3.0} & \textbf{Bio} & \textbf{LP} & \textbf{New3}  \\ \hline
% Graphene 4SATC        & 85.42           & 84.18          & 62.40          & 80.34          & 75.64 \\             

%\hline
%\end{tabular}
%\end{table}

\section{Support and Smatch}
We have shown in Table \ref{tab:support} that the average total support is highly correlated with the Smatch score. We performed statistical significant tests to support the given hypothesis. Below is the correlation between the "Normalised total support" (the total support normalised to the size of the graph) and the Smatch score, together with the p-value for each dataset:

\begin{itemize}
    \item AMR 2.0: Pearson correlation =0.60 , p-value = 2.7e-117
    \item AMR 3.0: Pearson correlation =0.49 , p-value = 2.6e-137
    \item BIO: Pearson correlation =0.55 , p-value = 0.0
    \item LP: Pearson correlation =0.56, p-value =3.4e-130
    \item New3: Pearson correlation = 0.73, p-value=5.1e-191
\end{itemize}
    
The overall correlation between the "Normalised total support" and the Smatch score, together with the p-value for all datasets is: Pearson correlation =0.67 , p-value=0.0

% keywords can be removed

% \bibliographystyle{unsrt}
% \bibliography{references}



\end{document}